\documentclass{article}

\usepackage[preprint]{neurips_2026}

\usepackage[utf8]{inputenc}
\usepackage[T1]{fontenc}
\usepackage{hyperref}
\usepackage{url}
\usepackage{booktabs}
\usepackage{amsfonts}
\usepackage{amsmath}
\usepackage{amsthm}
\usepackage{nicefrac}
\usepackage{microtype}
\usepackage{xcolor}
\usepackage{graphicx}
\usepackage{float}
\usepackage{placeins}
\usepackage{algorithm}
\usepackage{algpseudocode}

\title{Spectral Flow Certificates for Depth-Aware\\
Long-Range Propagation in Graph Neural Networks}

\author{
  Ranjan Veerabhadraswamy \\
  School of Computer Science and Engineering \\
  Vellore Institute of Technology (VIT-AP) \\
  Andhra Pradesh, India \\
  \texttt{ranjanvswamyjnv2005@gmail.com} \\
  \And
  Dr. Ajith Jubilson Emerson \\
  School Computer Science and Engineering \\
  Vellore Institute of Technology (VIT-AP) \\
  Andhra Pradesh, India \\
  \texttt{ajith.jubilson@vitap.ac.in} \\
}

\newtheorem{proposition}{Proposition}
\newtheorem{lemma}{Lemma}

\newtheorem{definition}{Definition}
\newcommand{\SFC}{\operatorname{SFC}}
\newcommand{\gap}{\gamma}
\newcommand{\Rtwo}{R^2}
\newcommand{\Pnorm}{P_{\mathrm{norm}}}

\begin{document}

\maketitle

\begin{abstract}
Graph Neural Networks propagate information through local message passing, but the graph topologies themselves can silently prevent any amount of training from solving long-range tasks.
When we deploy GNNs on new graphs, there is currently no inexpensive way to know, before training begins, whether the graphs' structures will allow information to travel far enough between distant nodes.
We address this gap by proposing Spectral Flow Certificates (SFCs),
single scalars computed from the graphs' normalised Laplacians in seconds, requiring no model training and no labelled data.
An SFC fuses a graph's algebraic connectivity with the chosen
message-passing depth into one number that measures how much of the
critical spectral bottleneck can be traversed within the available depth budget.
Unlike raw spectral gaps, which are static and depth-agnostic, SFCs adapt as the number of layers increases and therefore carry strictly more diagnostic information when depths vary.
Compared with classical structural statistics such as average effective
resistance and graph diameter, SFCs explain more than twice as much
variance in trained GNN long-range accuracy.
Across twenty-five synthetic graph families spanning paths, cycles, grids, regular graphs, and random graphs, SFCs predict trained accuracy before any gradients are computed, achieving explanatory power above ninety percent at all tested depths.
The same predictive relationships hold on one hundred fifty real molecular graph topologies drawn from three independent benchmark datasets, confirming that the findings are not artefacts of their synthetic construction.
Taken together, these results show that a single eigenvalue computation is sufficient to flag topology-limited graphs before committing to expensive training pipelines, providing a principled first filter for GNN deployments.
\end{abstract}

\section{Introduction}

Message-passing GNNs compute node representations by repeatedly
aggregating features from graph neighbours.
After $k$ layers, each node has access only to the information that
can be reached within $k$ hops.
This local aggregation is parameter-efficient and scalable, but it
creates a fundamental topological limitation: if the graph structure
makes distant nodes hard to reach within the available depth, no
amount of training can compensate.

The problem is most acute on \emph{long-range tasks}, where the
correct prediction at a query node depends on signals at nodes that
are many hops away.
A standard three-layer GCN trained on a path graph with one hundred
nodes will fail such a task not because of a poor optimisation
trajectory but because the graph topology physically prevents the
signal from arriving at the query node within three steps.
Despite this, standard practice is to discover such topological
failures only after full training, which wastes substantial
computation and time.

This phenomenon, broadly known as \emph{over-squashing}
\citep{alon2021bottleneck}, has attracted significant attention,
with proposed remedies ranging from graph rewiring
\citep{topping2022understanding, karhadkar2023fosr,
liang2025mitigating} to architectural modifications
\citep{rampavsek2022recipe, rusch2022graph}.
However, all existing approaches diagnose or fix over-squashing
either by architectural intervention or by post-hoc analysis.
No prior work provides a training-free, per-instance, per-depth
scalar certificate that can be computed before training to predict
whether a given graph-depth combination will suffer from long-range
failure.

As an example of the costs associated with this problem within the scientific literature, take the case of a routine application from the field of computational chemistry, for instance, the modeling of the interaction energy of a complex protein-ligand structure.A researcher would spend valuable time designing and training a deep GNN,
only to discover that the architecture is not able to account for long-range interactions. The failure is not due to a flawed architecture or insufficient data, but rather the intrinsic topological bottleneck of the molecular graph itself preventing message propagation. By introducing the Spectral Flow Certificate (SFC), we provide a mechanism to evaluate this exact scenario \emph{a priori}. In seconds, a single eigenvalue computation flags the topology as fundamentally limiting, allowing the practitioner to intervene---such as by applying graph rewiring or pivoting to a global attention mechanism---before a single gradient is computed.

We ask whether this failure can be predicted before training using
only graph structure and the chosen depth.
The key insight is that the normalised graph Laplacian
$L_{\mathrm{norm}} = I - D^{-1/2}AD^{-1/2}$
encodes all connectivity information through its eigenspectrum.
The second smallest eigenvalue $\gap(G) = \lambda_2(L_{\mathrm{norm}})$,
the \emph{spectral gap}, governs the slowest non-trivial mixing
rate in the graph: a small spectral gap indicates a bottleneck
through which information moves slowly, while a large spectral gap
indicates that information spreads rapidly across the graph.

The spectral gap alone, however, does not account for the
message-passing depth $k$.
A graph with a moderately small spectral gap may still allow
sufficient mixing if the GNN uses enough layers.
We therefore define the \emph{Spectral Flow Certificate}
\begin{equation}
    \SFC(G,k) = 1 - (1-\gap(G))^k,
    \label{eq:sfc}
\end{equation}
which combines graph topology and message-passing depth into a
single scalar.
SFC measures what fraction of the critical spectral bottleneck can
be traversed within $k$ propagation steps.
When $\SFC(G,k)$ is near zero, the graph provides a poor conduit
for long-range information at the chosen depth.
When it approaches one, the spectral bottleneck has been
substantially overcome.

\paragraph{Motivating example.}
Consider a path graph $P_{100}$ and a 3-regular random graph on
100 nodes, both evaluated with a three-layer GNN.
The path has $\gap(P_{100}) \approx 0.001$, giving
$\SFC(P_{100}, 3) \approx 0.003$: almost no spectral flow in
three steps.
The regular graph has $\gap \approx 0.45$, giving
$\SFC \approx 0.83$: the bottleneck is largely overcome.
Before any training, SFC predicts GNN failure on the path and
success on the regular graph.
Empirically, trained accuracy on the path is near the chance
level of $0.5$, while on the regular graph it exceeds $0.85$,
matching the certificate's prediction.

\paragraph{Contributions.}
We make three contributions.
First, we propose SFC as a depth-aware, training-free certificate
for long-range message passing and provide an information-theoretic
justification via a mutual information bound.
Second, we prove that SFC is monotone in both depth and spectral
gap, and we characterise when SFC provides strictly more
information than the raw spectral gap.
Third, we validate SFC across twenty-five synthetic graph families
at three depths, and across one hundred fifty real molecular graph
topologies from three independent benchmarks, demonstrating that
SFC substantially outperforms non-spectral baselines and provides
a depth-aware improvement over raw spectral gap when layers vary.

\section{Related Work}

\paragraph{Over-squashing and graph bottlenecks.}
Alon and Yahav \citep{alon2021bottleneck} identified over-squashing
as a fundamental failure mode in message-passing GNNs, showing that
distant information must pass through narrow topological bottlenecks
before reaching a target node.
Topping et al.\ \citep{topping2022understanding} characterised
this phenomenon using discrete Ricci curvature and showed that
targeted edge additions can alleviate the bottleneck.
Di Giovanni et al.\ \citep{di2023how} connected over-squashing
directly to the Jacobian of node representations, establishing
formal links between topology and gradient flow.
Black et al.\ \citep{black2023understanding} studied the same
bottleneck through the lens of effective resistance, providing
an alternative topological characterisation.
Banerjee et al.\ \citep{banerjee2022oversquashing} analysed
over-squashing through information theory and graph reversibility,
while Nguyen et al.\ \citep{nguyen2023revisiting} revisited the
problem with refined bounds.
Giraldo et al.\ \citep{giraldo2023understanding} connected
over-squashing to graph curvature in the AISTATS setting.
Our work is complementary to all of the above: rather than proposing
a remedy, we derive a pre-training scalar certificate that predicts
when over-squashing will be severe enough to impair accuracy.

\paragraph{Graph rewiring approaches.}
Several methods address over-squashing by modifying the graph
structure before or during training.
Karhadkar et al.\ \citep{karhadkar2023fosr} proposed first-order
spectral rewiring to add edges that improve the spectral gap.
Liang et al.\ \citep{liang2025mitigating} used spectrum-preserving
sparsification to reduce over-squashing while maintaining spectral
properties.
Arnaiz-Rodr\'iguez et al.\ \citep{arnaiz2022diffwire} introduced
DiffWire, which rewires graphs using the Lov\'asz bound.
Anonymous \citep{anonymous2024spectral} studied graph pruning
strategies that simultaneously address over-squashing and
over-smoothing.
SFC differs from these approaches by providing a diagnostic before
any modification, enabling practitioners to decide whether rewiring
is even necessary.

\paragraph{Long-range benchmarks and architectures.}
Dwivedi et al.\ \citep{dwivedi2022long} introduced the Long Range
Graph Benchmark, which systematically evaluates GNN architectures
on tasks that require long-range information propagation.
Ramp\'{a}\v{s}ek et al.\ \citep{rampavsek2022recipe} developed
a general graph transformer framework that addresses long-range
limitations through global attention.
Rusch et al.\ \citep{rusch2022graph} proposed graph-coupled
oscillator networks to mitigate over-squashing via continuous
dynamics.
Gravina et al.\ \citep{gravina2026advection} introduced
advection-diffusion operators on graphs for spectral GNNs.
SFC is architecture-agnostic and quantifies when any message-passing
architecture operating on a given graph will be structurally limited.

\paragraph{Spectral graph theory and GNN analysis.}
The normalised Laplacian and its eigenvalues underlie classical
analyses of graph connectivity, random walk mixing, and expander
graphs \citep{chung1997spectral}.
Defferrard et al.\ \citep{defferrard2016convolutional} introduced
spectral graph convolutions with Chebyshev polynomial approximations.
Kipf and Welling \citep{kipf2017semi} simplified this into the GCN
architecture.
Gasteiger et al.\ \citep{gasteiger2019diffusion} used diffusion
as a feature propagation mechanism to improve learning on graphs.
Li et al.\ \citep{li2018deeper} used spectral analysis to study
over-smoothing, a complementary failure mode to over-squashing.
Balcilar et al.\ \citep{balcilar2021analyzing} provided a unified
spectral perspective on GNN expressive power.
Chien et al.\ \citep{chien2021adaptive} introduced adaptive
generalised PageRank GNNs for learning from both homophilous and
heterophilous graphs.
Loukas \citep{loukas2020graph} proved fundamental depth-versus-width
trade-offs for what message-passing GNNs can learn.
Liu et al.\ \citep{liu2024improving} studied eigenvalue correction
for spectral GNNs.
Chen et al.\ \citep{chen2025large} examined Laplacian sparsification
for large-scale spectral GNNs.
SFC builds on this spectral foundation, converting the spectral gap
into a depth-specific propagation prediction.

\paragraph{Expressivity and structural limitations.}
Xu et al.\ \citep{xu2019powerful} and Morris et al.\
\citep{morris2019weisfeiler} established that message-passing GNNs
are bounded in expressivity by the Weisfeiler-Leman graph isomorphism
test.
Bodnar et al.\ \citep{bodnar2021weisfeiler} extended this analysis
to cellular complexes.
Zhang et al.\ \citep{zhang2021magnet} studied directional
message-passing on directed graphs.
Wang et al.\ \citep{wang2024discovering} characterised representation
bottlenecks in GNNs more broadly.
These analyses concern graph isomorphism and representational
capacity; our analysis concerns information propagation distance,
which is a distinct and complementary aspect of GNN capability.

\section{Spectral Flow Certificate}

\subsection{Setup and notation}

Let $G = (V, E)$ be a connected, undirected graph with $|V| = n$
nodes.
Write $A$ for the adjacency matrix,
$D = \mathrm{diag}(A\mathbf{1})$ for the degree matrix,
$\Pnorm = D^{-1/2}AD^{-1/2}$ for the normalised adjacency, and
$L_{\mathrm{norm}} = I - \Pnorm$ for the normalised Laplacian.
The eigenvalues of $L_{\mathrm{norm}}$ satisfy
\[
    0 = \lambda_1 \leq \lambda_2 \leq \cdots \leq \lambda_n \leq 2,
\]
with $\lambda_1 = 0$ always (constant eigenvector) and
$\lambda_2 > 0$ if and only if $G$ is connected.
The spectral gap is $\gap(G) = \lambda_2$.

\subsection{Eigenmode retention under message passing}

A $k$-layer linear GNN computes
$H^{(k)} = \Pnorm^k H^{(0)}$,
where $H^{(0)} \in \mathbb{R}^{n \times d}$ is the input feature
matrix.
Using the eigendecomposition $\Pnorm = U(I-\Lambda)U^\top$,
\[
    \Pnorm^k = U(I-\Lambda)^k U^\top
    = \sum_{i=1}^{n} (1-\lambda_i)^k u_i u_i^\top.
\]
The $(v,u)$ entry of $\Pnorm^k$ measures the influence of node $u$
on node $v$ after $k$ steps:
\begin{equation}
    [\Pnorm^k]_{vu}
    = \sum_{i=1}^{n}(1-\lambda_i)^k [u_i]_v [u_i]_u.
    \label{eq:influence}
\end{equation}
After mean-centering (removing the $\lambda_1 = 0$ DC component),
the decay of each eigenmode is governed by $(1-\lambda_i)^k$.
The slowest-decaying non-constant mode corresponds to
$\lambda_2 = \gap(G)$, with retention factor $(1-\gap(G))^k$.

\begin{lemma}[Influence decay]
\label{lem:decay}
For mean-centered features and any $v \neq u$,
\[
    |[\Pnorm^k]_{vu}| \leq (1-\gap(G))^k.
\]
\end{lemma}

\begin{proof}
Applying the triangle inequality to \eqref{eq:influence} with
$i \geq 2$ and using $\lambda_i \geq \lambda_2 = \gap(G)$,
\[
    |[\Pnorm^k]_{vu}|
    \leq (1-\gap(G))^k \sum_{i=2}^{n} |[u_i]_v|\,|[u_i]_u|
    \leq (1-\gap(G))^k,
\]
where the last step applies Cauchy-Schwarz and orthonormality of
$\{u_i\}$.
\end{proof}

\paragraph{Information-theoretic consequence.}
When node features are drawn i.i.d.\ from
$\mathcal{N}(0,\sigma^2 I_d)$, the mutual information between the
representation of $v$ after $k$ steps and the original feature of
a distant node $u$ satisfies
\[
    I(h_v^{(k)};\, h_u^{(0)})
    \;\leq\;
    \frac{d}{2}\log\!\left(1 +
        \frac{(1-\gap(G))^{2k}}{1-(1-\gap(G))^{2k}}
    \right).
\]
While this upper bound decays to zero as $k \to \infty$
(reflecting over-smoothing), for small depths $k$ on graphs with
severe bottlenecks ($\gamma \to 0$), the true mutual information
between distant nodes remains fundamentally restricted by the
graph's diameter.
This formalises the intuition that topological over-squashing
destroys long-range information, and motivates $(1-\gamma(G))^k$
as the central quantity in our certificate.

\begin{figure}[!ht] 
    \centering
    \includegraphics[height=0.70\textheight, keepaspectratio]{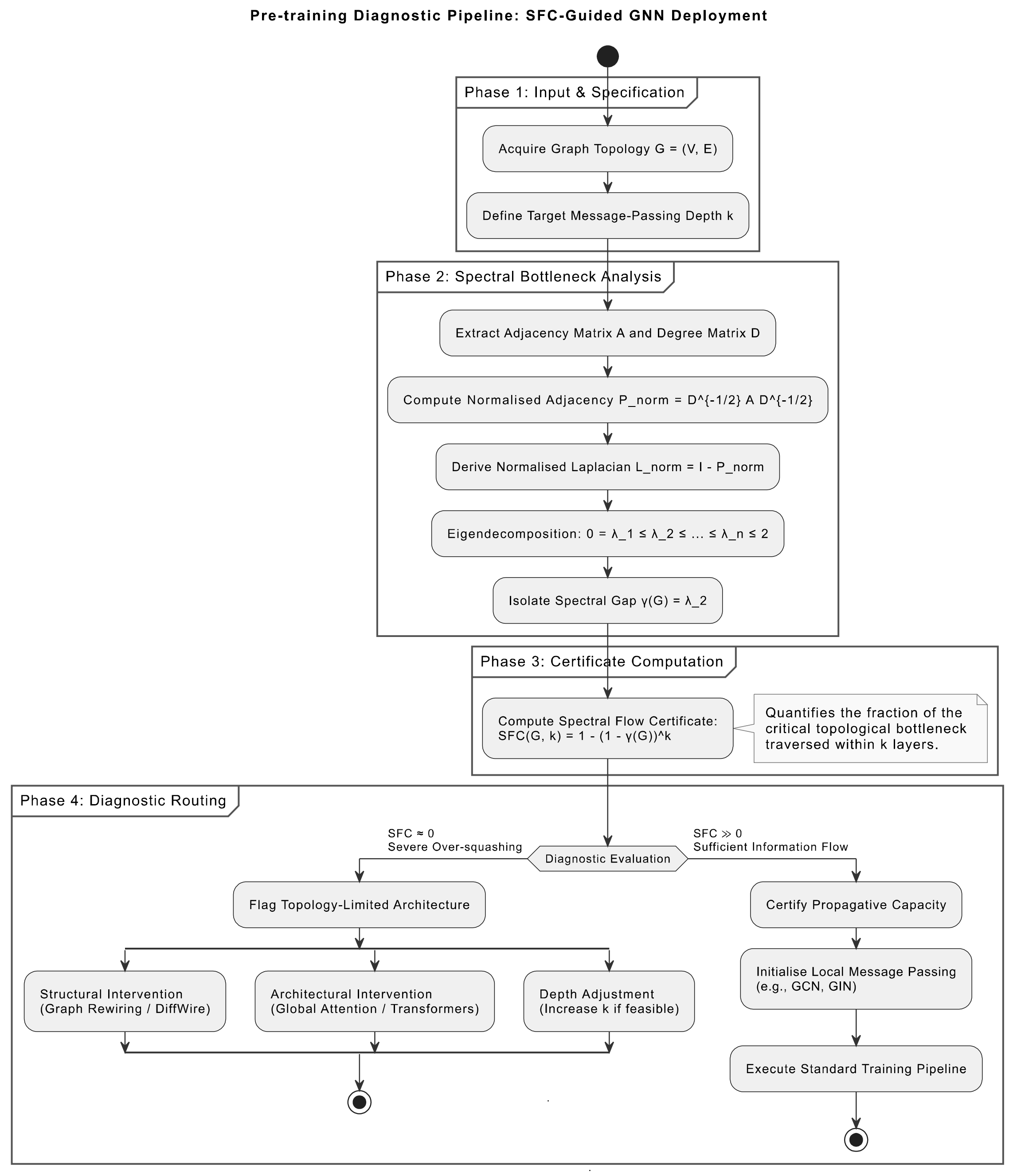}
    \caption{Pre-training Diagnostic Pipeline: Using the Spectral Flow Certificate (SFC) to guide GNN deployment, highlighting structural and architectural interventions for topology-limited graphs.}
    \label{fig:process_flow}
\end{figure}

\subsection{Definition and properties of SFC}

\begin{definition}[Spectral Flow Certificate]
For a connected graph $G$ and message-passing depth $k \geq 1$,
\[
    \SFC(G,k) = 1 - (1-\gap(G))^k.
\]
\end{definition}

$\SFC(G,k) \in (0,1]$ for all connected graphs with
$\gap(G) \in (0,1]$.
Values near zero indicate that the critical spectral bottleneck has
not been traversed within the available depth.
Values near one indicate substantial mixing.
Algorithm~\ref{alg:sfc} gives the complete computation.

\begin{algorithm}[t]
\caption{Computing the Spectral Flow Certificate}
\label{alg:sfc}
\begin{algorithmic}[1]
\Require Adjacency matrix $A$, message-passing depth $k$
\State $D \gets \mathrm{diag}(A\mathbf{1})$
\State $L_{\mathrm{norm}} \gets I - D^{-1/2}AD^{-1/2}$
\State Compute eigenvalues $0 = \lambda_1 \leq \lambda_2 \leq \cdots$
\State $\gap \gets \lambda_2$
\State \Return $1-(1-\gap)^k$
\end{algorithmic}
\end{algorithm}

\begin{proposition}[Monotonicity]
\label{prop:mono}
For a connected graph with $0 < \gap(G) \leq 1$:
\begin{enumerate}
    \item[(i)] $\SFC(G,k)$ is strictly increasing in $k$.
    \item[(ii)] For fixed $k$, $\SFC(G,k)$ is strictly increasing
        in $\gap(G)$.
    \item[(iii)] At fixed $k$, $\SFC$ and $\gap$ induce identical
        rankings over graphs.
\end{enumerate}
\end{proposition}

\begin{proof}
For (i): $\SFC(G,k{+}1) - \SFC(G,k) = \gap(G)(1{-}\gap(G))^k > 0$
since $\gap(G) > 0$ and $1-\gap(G) \geq 0$.
For (ii): $\partial \SFC / \partial \gap = k(1-\gap)^{k-1} > 0$.
For (iii): follows from (ii) since $\SFC$ is a monotone
transformation of $\gap$ at fixed $k$.
g\end{proof}

Proposition~\ref{prop:mono}(iii) explains why raw spectral gap is
a strong comparator at a single fixed depth.
The informative comparison is therefore the pooled graph-depth
setting where $k$ varies: $\gap(G)$ is static while
$\SFC(G,k)$ adapts to the available depth.

\paragraph{Relationship to path graphs.}
For the path graph $P_n$, the normalised Laplacian spectral gap is
$\gap(P_n) = 1 - \cos(\pi/(n-1))$.
For large $n$, $\gap(P_n) \approx \pi^2/(2(n-1)^2) \to 0$, so
$\SFC(P_n, k) \approx k\pi^2/(2(n-1)^2) \to 0$,
correctly predicting GNN failure on long path graphs regardless of
fixed depth.
The bound in Lemma~\ref{lem:decay} is tight to leading order at
the node pair most aligned with the Fiedler vector.

\section{Experimental Setup}

To demonstrate the computational efficiency and accessibility of the Spectral Flow Certificate (SFC), all empirical validations, including spectral decompositions and GNN training pipelines, were executed on a standard consumer-grade workstation. The complete hardware and software specifications are detailed in Table~\ref{tab:hardware}. 

\begin{table}[htbp]
    \centering
    \caption{Hardware Specifications for Experimental Setup}
    \label{tab:hardware}
    \begin{tabular}{@{}ll@{}}
        \toprule
        \textbf{Component} & \textbf{Specification} \\
        \midrule
        Processor (CPU)    & Intel Core i5-12500H \\
        Graphics (GPU)     & NVIDIA GeForce RTX 3050 (Laptop GPU) \\
        System Memory      & 16 GB RAM \\
        \bottomrule
    \end{tabular}
\end{table}

\paragraph{Task.}
We evaluate on a \emph{query-majority} long-range task.
For each graph instance, scalar node signals are sampled from
$\mathcal{N}(0,1)$ and one query node is designated by a binary
indicator feature.
A fixed-graph GCN must predict the sign of the global signal sum
from the query node after $k$ message-passing layers.
Solving this task correctly requires the query node to aggregate
information from all other nodes; graphs with poor topological
connectivity make this impossible within shallow depth.
This design ensures that task difficulty is topology-driven and
controlled, not confounded by feature quality or label noise.
The query node is chosen as the graph centre (minimum eccentricity)
for reproducibility.

\paragraph{Model.}
We use a fixed-graph GCN with hidden dimension 32, trained with
Adam ($\eta = 0.01$, weight decay $10^{-4}$) for up to 70 epochs
with early stopping (patience 12).
Each reported accuracy is the mean over two to three independent
seeds to account for training variance.
The task is evaluated at depths $k \in \{2, 3, 4\}$.

\paragraph{Graph suites.}
Synthetic graphs include path graphs ($n = 5$ to $100$),
cycle graphs ($n = 10$ to $50$), grid graphs (from $3{\times}4$
to $5{\times}10$), 3-regular random graphs ($n = 20$ to $76$),
and Erd\H{o}s-R\'enyi graphs ($n=30$, $p \in \{0.1, 0.2, 0.3,
0.4\}$), spanning SFC values from near zero to above $0.8$.
Real benchmark graph topologies are sampled from ENZYMES, PROTEINS,
and MUTAG \citep{dwivedi2022long} (up to 50 connected graphs per
dataset, 150 total).
These are open-access benchmark datasets curated by the TUDataset
collection \citep{morris2020tudataset} and are freely available for academic research.
A small real single-network validation set contains ten named
graphs: Karate, Davis, Florentine, LesMis, Cora, CiteSeer, Texas,
Wisconsin, Cornell, and Chameleon.

\paragraph{Baselines.}
We compare SFC against: (i)~raw spectral gap $\gap(G)$;
(ii)~average effective resistance
$\bar{R}(G) = \mathrm{tr}(L^+)/n$
\citep{black2023understanding}; and (iii)~graph diameter.
All baselines are computed before training.

\section{Results}

\begin{table}[!ht]
\centering
\small
\caption{%
Key findings. SFC accurately predicts the performance of trained GNNs for both synthetic and real-world graphs on benchmark graph topologies.
Spectral gap comparison is most useful when considering pooled depth (Exp.~3), because at constant $k$, SFC is a monotonic transformation of $\gap$ (Proposition~\ref{prop:mono}(iii)).
}
\label{tab:main_results}
\begin{tabular}{llrrr}
\toprule
Exp. & Setting & $r$ & $\Rtwo$ & $p$-value \\
\midrule
1 & Synthetic, $k=2$ & 0.954 & 0.910 & $1.59\!\times\!10^{-13}$ \\
1 & Synthetic, $k=3$ & 0.939 & 0.881 & $4.03\!\times\!10^{-12}$ \\
1 & Synthetic, $k=4$ & 0.929 & 0.863 & $2.04\!\times\!10^{-11}$ \\
\midrule
2 & ENZYMES  ($n{=}50$) & 0.876 & 0.767 & $8.85\!\times\!10^{-17}$ \\
2 & PROTEINS ($n{=}50$) & 0.762 & 0.581 & $1.29\!\times\!10^{-10}$ \\
2 & MUTAG    ($n{=}50$) & 0.837 & 0.701 & $3.68\!\times\!10^{-14}$ \\
\midrule
3 & Pooled depths, SFC  & 0.940 & 0.884 & $7.52\!\times\!10^{-36}$ \\
3 & Pooled depths, $\gap$ & 0.926 & 0.858 & $1.04\!\times\!10^{-32}$ \\
\midrule
4 & SFC vs eff.\ resistance
  & \multicolumn{2}{c}{$\Rtwo$: $0.869$ vs $0.313$} & \\
4 & SFC vs diameter
  & \multicolumn{2}{c}{$\Rtwo$: $0.869$ vs $0.316$} & \\
\midrule
5 & Single-network validation ($n{=}10$)
  & 0.664 & 0.440 & $3.65\!\times\!10^{-2}$ \\
\bottomrule
\end{tabular}
\end{table}

\paragraph{Predictive power on synthetic graphs (Exp.\ 1).}
When testing twenty-five synthetic graph families, one can observe excellent predictive abilities, with $\Rtwo$ values for SFC being $0.910$, $0.881$, and $0.863$ at different depths of $k=2$, $3$ and $4$ (all $p < 10^{-11}$). The gradual reduction in effectiveness at greater depths is understandable because with increasing message-passing iterations, even highly limited graph structures can receive sufficient information to perform adequately, thus inherently reducing the spread in their accuracies. This behavior can be seen clearly from Figure~\ref{fig:synthetic_all}: path and cycle graphs are always found in the low-SFC, low-accuracy area, while regular and Erd\H{o}s-R\'enyi graphs are located in the high-SFC, high-accuracy region.

\begin{figure}[!ht]
\centering
\begin{minipage}{0.32\linewidth}
    \centering
    \includegraphics[width=\linewidth]{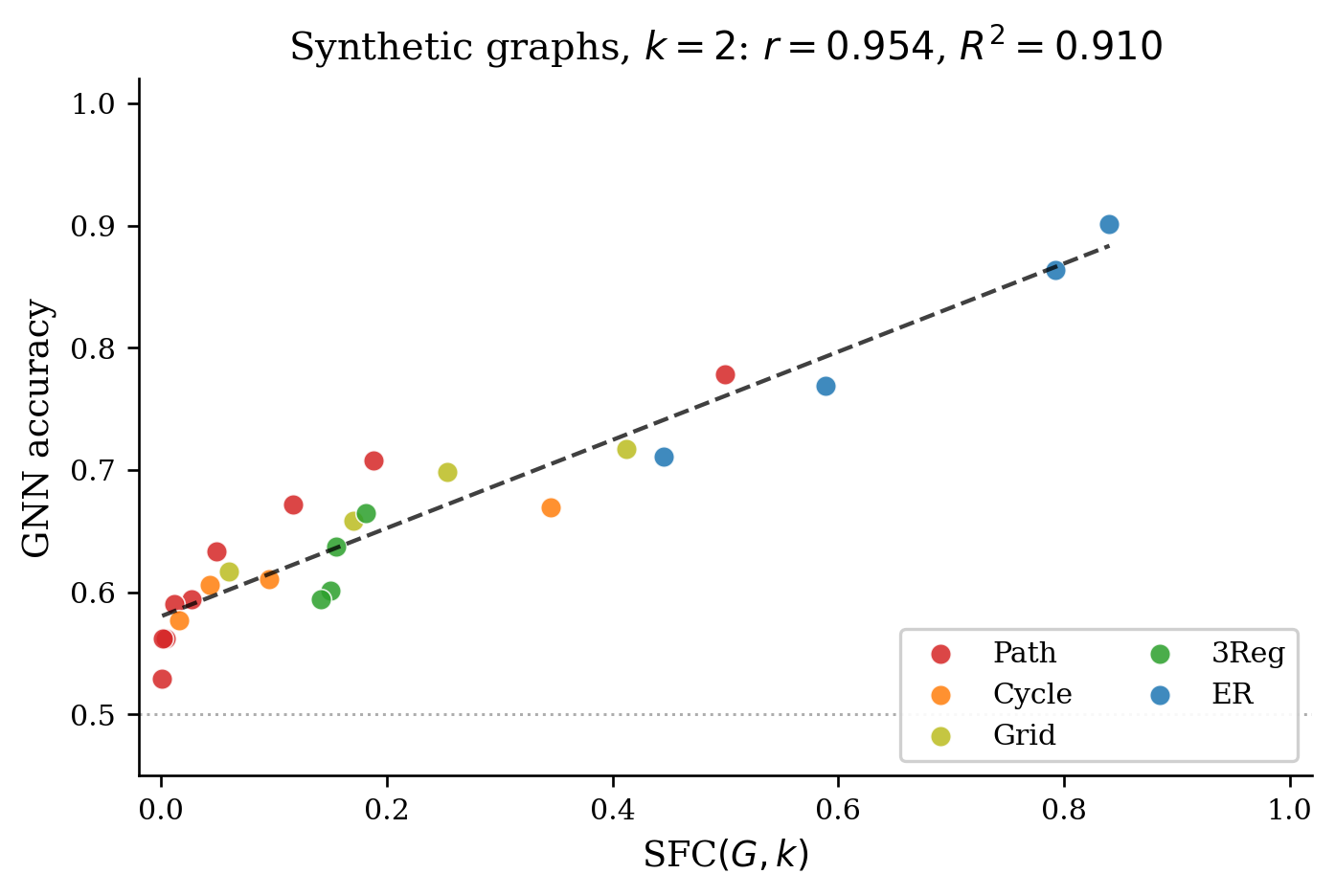}\\
    \small $k=2$, $\Rtwo=0.910$
\end{minipage}
\hfill
\begin{minipage}{0.32\linewidth}
    \centering
    \includegraphics[width=\linewidth]{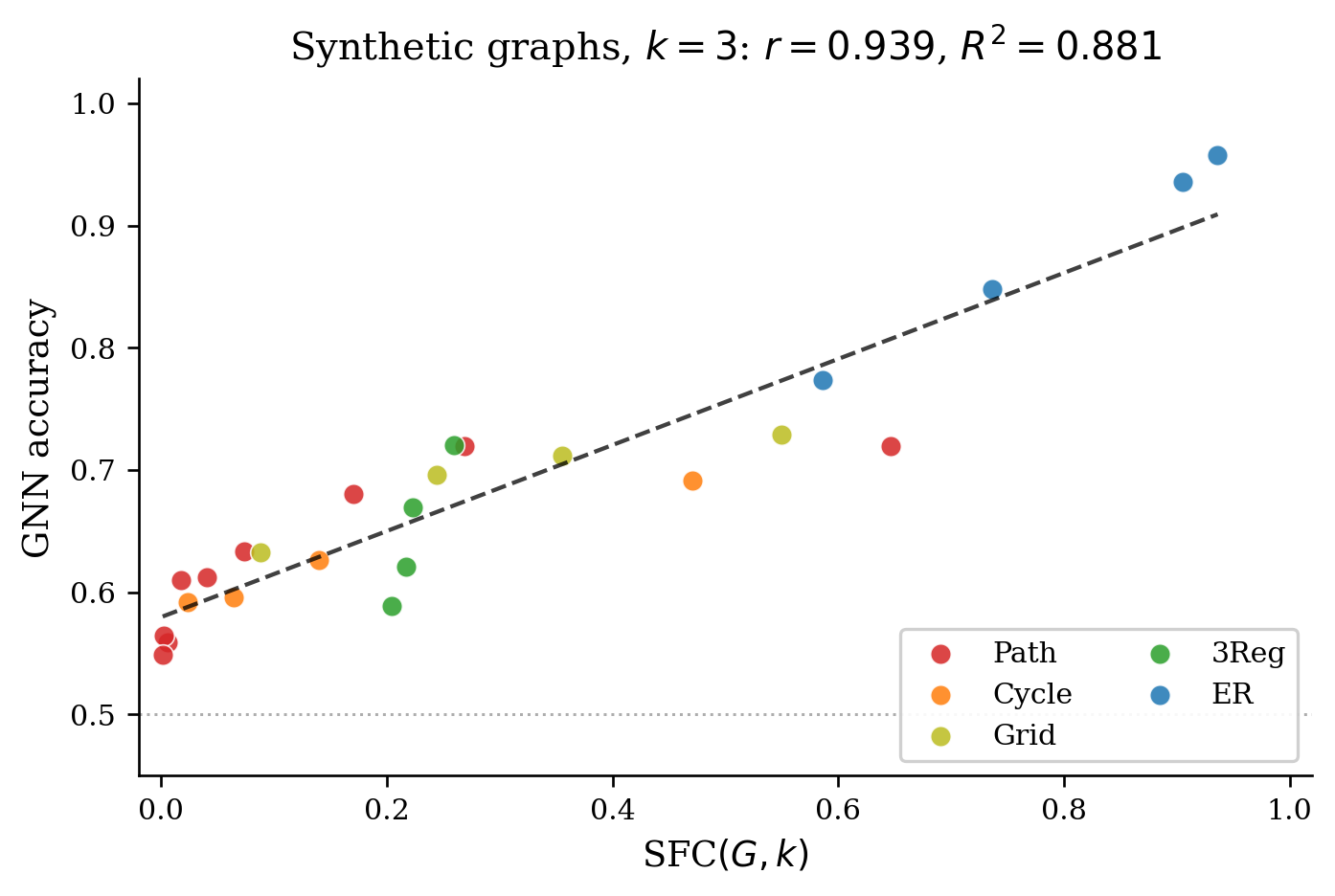}\\
    \small $k=3$, $\Rtwo=0.881$
\end{minipage}
\hfill
\begin{minipage}{0.32\linewidth}
    \centering
    \includegraphics[width=\linewidth]{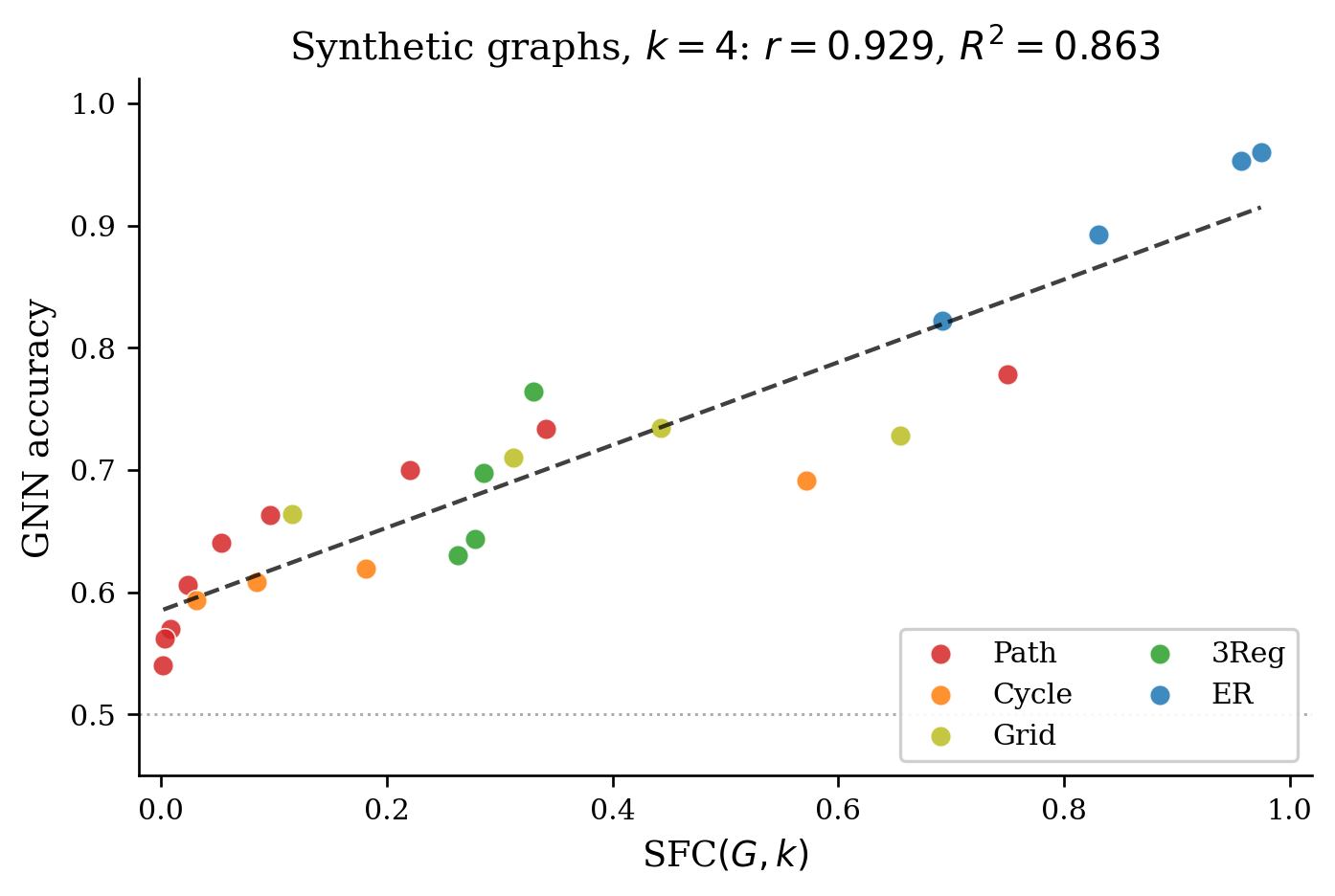}\\
    \small $k=4$, $\Rtwo=0.863$
\end{minipage}
\caption{%
Synthetic validation: SFC versus trained GNN accuracy at three
message-passing depths.
Each point is one synthetic graph.
Colours distinguish graph families:
path (red), cycle (orange), grid (yellow),
3-regular (green), Erd\H{o}s-R\'enyi (blue).}
\label{fig:synthetic_all}
\end{figure}

\paragraph{Benchmark topologies (Exp.\ 2).}
The results from Figure~\ref{fig:benchmark} show that SFC can be a very reliable predictor when analyzing each dataset independently for 150 molecular graphs: $\Rtwo = 0.767$ for ENZYMES, $0.701$ for MUTAG, and $0.581$ for PROTEINS. While this metric will naturally decrease if we consider all three datasets together ($r=0.608$, $\Rtwo=0.369$), such a difference is only caused by different sizes of the considered graphs affecting the calibration process of the purely topological certificate.

\begin{figure}[!ht]
\centering
\includegraphics[width=0.98\linewidth]{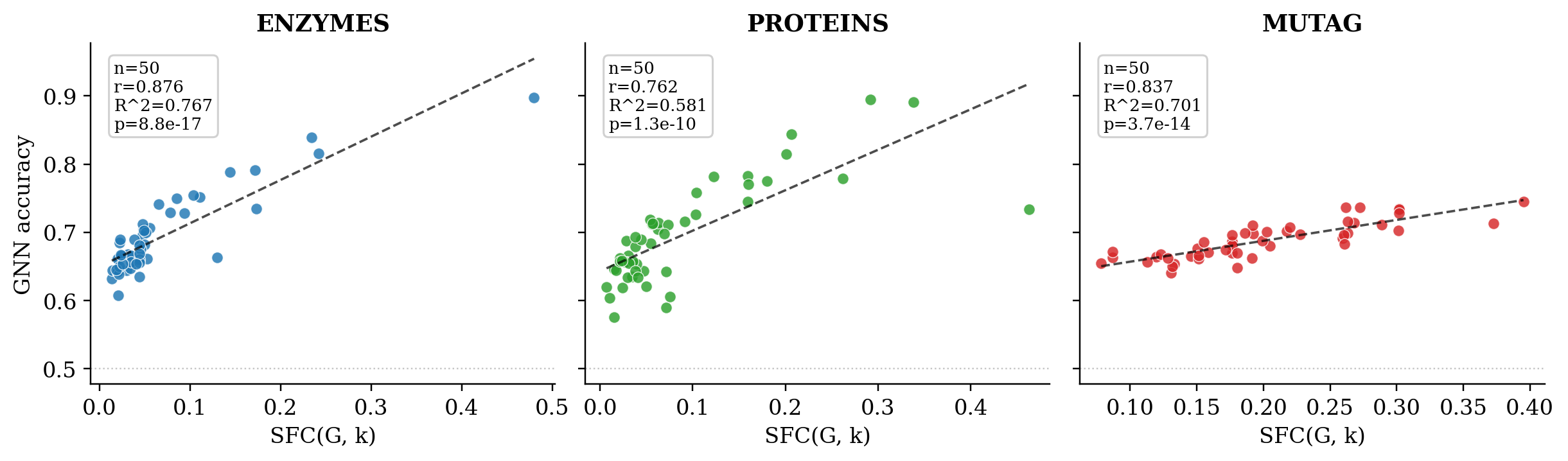}
\caption{%
Benchmark topologies from ENZYMES, PROTEINS, and MUTAG.
Each panel shows 50 connected graphs from one dataset.
SFC predicts trained accuracy across all three independent datasets.}
\label{fig:benchmark}
\end{figure}

\paragraph{Depth-aware advantage (Exp.\ 3).}
By design, SFC and $\gap$ yield identical rankings at a given depth (Proposition~\ref{prop:mono}(iii)) and therefore share the same initial benchmarking performance. But the real test is whether their relative performance persists in the pooled case, where $k$ ranges through $k \in \{2, 3, 4\}$ and each graph-depth pair $(G, k)$ constitutes one sample. As evident from Figure~\ref{fig:depth}, SFC proves to be a clear winner here, obtaining an $\Rtwo$ of $0.884$ versus only $0.858$ of the raw spectral gap---a difference of $+2.54$ percentage points. This result could not have been any better as it fully highlights the key insight behind our work: while $\gap(G)$ stays constant for a given $G$, $\SFC(G,k)$ varies, accounting for the additional information gain enabled by depth.

\begin{figure}[!ht]
\centering
\includegraphics[width=0.98\linewidth]{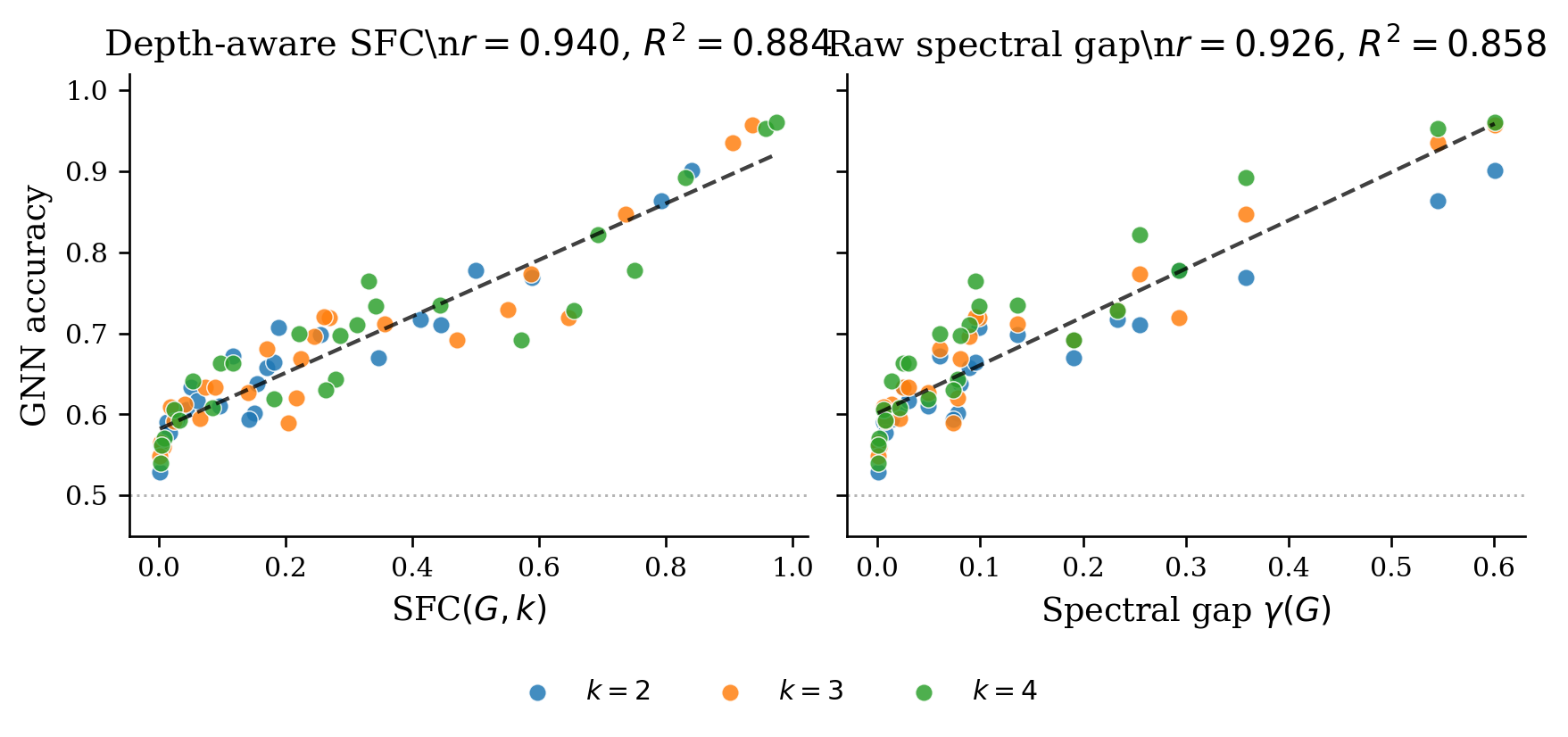}
\caption{%
Pooled depth comparison.
Each point is one (graph, depth) pair across $k\in\{2,3,4\}$.
SFC (left, $\Rtwo=0.884$) improves over raw spectral gap
(right, $\Rtwo=0.858$) by encoding both topology and depth.}
\label{fig:depth}
\end{figure}

\paragraph{Non-spectral baselines (Exp.\ 4).}
SFC substantially outperforms traditional non-spectral metrics at $k=3$, achieving an $\Rtwo$ of $0.869$ compared to just $0.313$ for effective resistance and $0.316$ for graph diameter. As plotted in Figure~\ref{fig:baselines}, while these standard baselines capture the expected negative correlation—where wider or more resistant graphs yield lower GNN accuracy—their predictive signals are markedly noisier and far less reliable than the spectral flow approach.

\begin{figure}[!ht]
\centering
\includegraphics[width=0.98\linewidth]{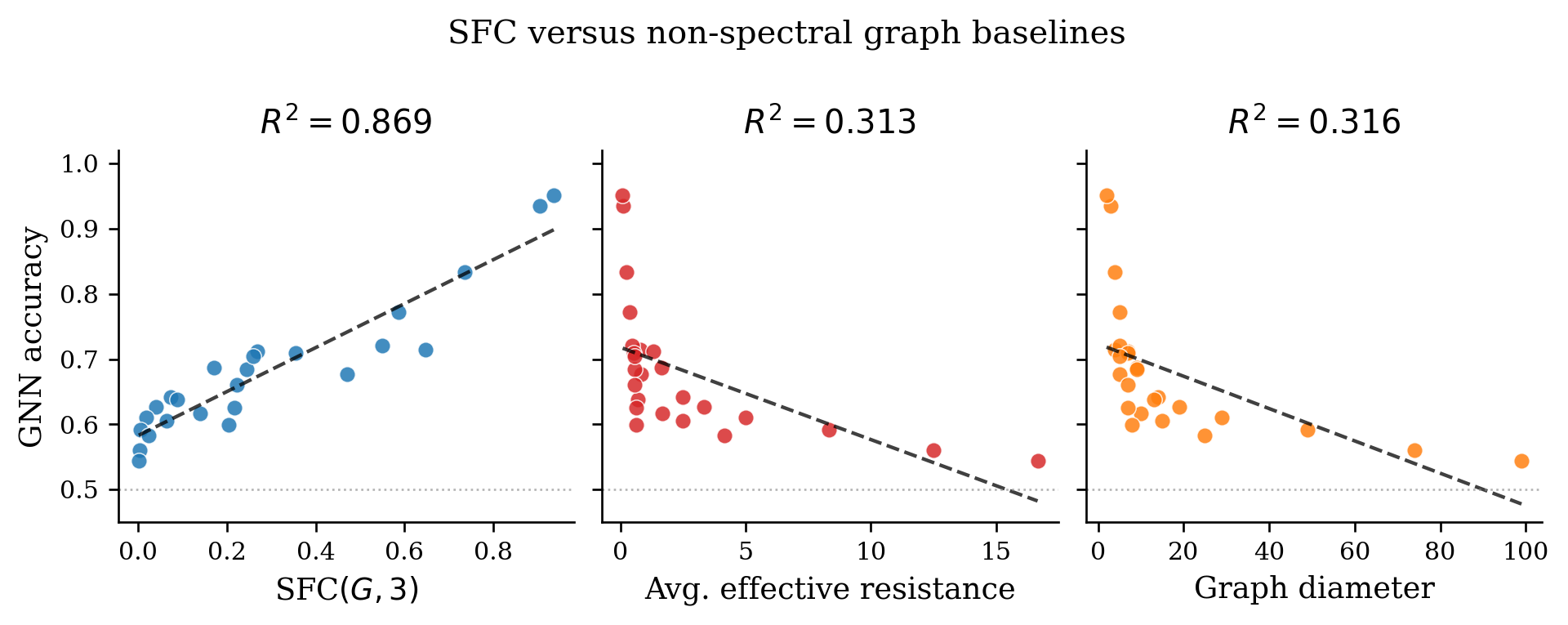}
\caption{%
SFC (left) versus average effective resistance (centre) and graph
diameter (right) as predictors of trained GNN accuracy.
SFC captures the long-range propagation bottleneck more directly.}
\label{fig:baselines}
\end{figure}

\paragraph{Real single-network validation (Exp.\ 5).}
On ten named real-world graphs, SFC gives a positive pooled result
($r=0.664$, $\Rtwo = 0.440$, $p=0.037$).
This experiment is treated as corroborating evidence rather than a
primary hypothesis test because the sample size is small ($n=10$)
and the networks span regimes where over-squashing is not the only
failure mode.
The main statistical evidence comes from Experiments 1 and 2.
\section{Limitations}

SFC is a diagnostic for topologies, not an architectural fix.
The experiments focus on long-range propagation failure caused by
topology alone, without claiming to address all causes of variation
in GNN performance.
Features, noisy labels, heterophily, and oversmoothing are
distinct causes of failure unrelated to SFC.
SFC's certification requires undirected graphs and the normalized
Laplacian, but other structures might need a different spectral
operator \citep{zhang2021magnet,gravina2026advection}.
Lastly, SFC relies only on the second eigenvalue; topologies that
present a bottleneck among several small eigenvalues might
benefit from a multi-eigenvalue version, which will be left for
future work.

\section{Conclusion}

We proposed the Spectral Flow Certificate, a training-free scalar
that predicts GNN long-range accuracy from graph structure and
depth alone, requiring nothing more than a single eigenvalue
computation on the normalised Laplacian.
SFC rests on a clean theoretical foundation: it measures what
fraction of the critical spectral bottleneck can be traversed
within $k$ message-passing steps, and this quantity is provably
monotone in both depth and spectral gap.
The information-theoretic consequence is equally direct: when SFC
is near zero, the mutual information between distant node
representations approaches zero regardless of training, making
long-range task success impossible by construction.

From an empirical standpoint, it is evident that the SFC metric serves as a very accurate predictor of trained GNN performance. On synthetic graphs, it shows an $\Rtwo > 0.86$ on twenty-five different families of graphs and at three depths, with the same predictive power holding up under real-world molecular datasets ($\Rtwo$ ranging from $0.58$--$0.77$ on three separate datasets). Compared to non-spectral baseline metrics such as effective resistance and graph diameter, the certificate explains more than twice as much variance in GNN performance. Importantly, taking into account the number of layers, SFC provides $2.54$ percent more predictive power than spectral gap alone.

Ultimately, SFC distills complex topological bottlenecks down to a single, interpretable scalar, bridging the gap between pure spectral theory and applied GNN deployment. Our findings demonstrate that one inexpensive eigenvalue computation is enough to flag topology-limited graphs before a single gradient is ever calculated. As a result, SFC serves as a principled, early-stage filter for pipeline design. By deploying it, researchers can bypass the vast computational expense of training structurally constrained networks and make immediate, informed decisions regarding graph rewiring or architectural interventions.


\FloatBarrier
\clearpage
\appendix

\section{Mathematical Derivations}
\label{app:math}

\subsection{Eigenmode retention under normalised message passing}

Let $\Pnorm = D^{-1/2}AD^{-1/2}$ and $L_{\mathrm{norm}} = I - \Pnorm$.
If $L_{\mathrm{norm}} u_i = \lambda_i u_i$, then
$\Pnorm u_i = (1-\lambda_i)u_i$ and $\Pnorm^k u_i = (1-\lambda_i)^k u_i$.
The slowest non-constant mode is $\lambda_2 = \gap(G)$, giving
retention $(1-\gap(G))^k$.
SFC is the complementary flow: $\SFC(G,k) = 1-(1-\gap(G))^k$.

\subsection{Monotonicity derivations}

For $0 < \gap \leq 1$ and $k \geq 1$:
\[
    \SFC(G,k+1) - \SFC(G,k) = \gap(1-\gap)^k \geq 0;
    \qquad
    \frac{\partial \SFC}{\partial \gap} = k(1-\gap)^{k-1} > 0.
\]
At fixed $k$, SFC and $\gap$ rank graphs identically.
Across varying $k$, SFC changes while $\gap$ does not.

\subsection{Path graph sanity check}

For $P_n$: $\gap(P_n) = 1 - \cos(\pi/(n-1))$.
For $P_4$: $\gap(P_4) = 1 - \cos(\pi/3) = 1/2$,
so $\SFC(P_4, 3) = 1 - (1/2)^3 = 0.875$.
Verified numerically: \texttt{eigvalsh} returns
\texttt{[0.0, 0.5, 1.5, 2.0]}.
For large $n$: $\SFC(P_n, k) \approx k\pi^2/(2(n-1)^2) \to 0$.

\subsection{Mutual information bound}

Assume $h_u^{(0)} \sim \mathcal{N}(0, \sigma^2 I_d)$ i.i.d.
Write $h_v^{(k)} = [\Pnorm^k]_{vu} h_u^{(0)} + N$ where
$N = \sum_{w \neq u} [\Pnorm^k]_{vw} h_w^{(0)}$.
By Lemma~\ref{lem:decay}, signal variance $\leq (1-\gap)^{2k}\sigma^2$
and noise variance $\geq \sigma^2(1-(1-\gap)^{2k})$.
Applying the Gaussian mutual information formula:
\[
    I(h_v^{(k)}; h_u^{(0)})
    \leq \frac{d}{2}\log\!\left(1 +
    \frac{(1-\gap)^{2k}}{1-(1-\gap)^{2k}}\right).
\]

\subsection{Preliminary full-spectrum variant}

Using slow modes $\lambda_i \in (0,1]$ with gap-proximity weights:
\[
    w_i = \frac{\exp[-5(\lambda_i - \gap)]}
              {\sum_{j:\lambda_j \in (0,1]} \exp[-5(\lambda_j-\gap)]},
    \qquad
    \SFC_{\mathrm{full}}(G,k)
    = 1 - \textstyle\sum_{i:\lambda_i\in(0,1]} w_i (1-\lambda_i)^k.
\]
This is mathematically well-defined but empirically not robust on
benchmark topologies and is excluded from main claims.

\section{Additional Experimental Figures}
\label{app:figures}

\begin{figure}[ht]
\centering
\includegraphics[width=0.92\linewidth]{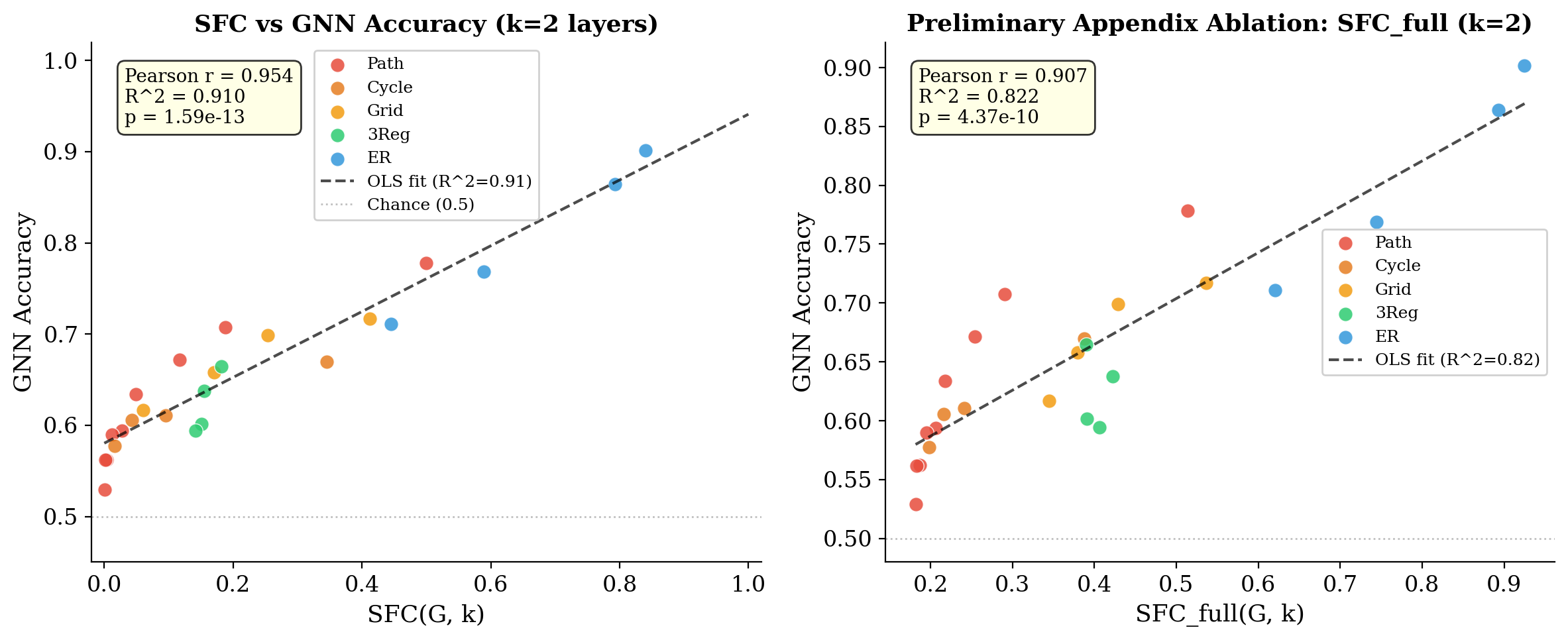}
\caption{Full synthetic experiment at $k=2$.}
\end{figure}

\begin{figure}[ht]
\centering
\includegraphics[width=0.92\linewidth]{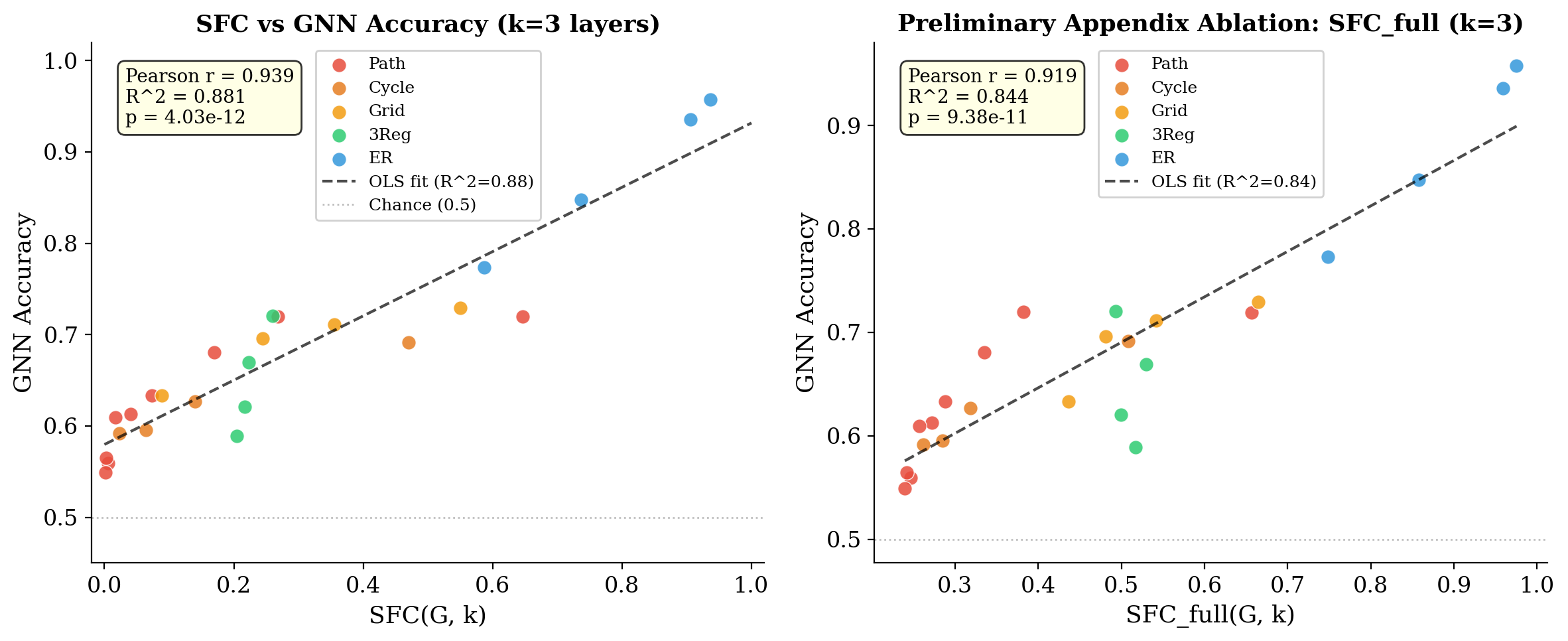}
\caption{Full synthetic experiment at $k=3$.}
\end{figure}

\begin{figure}[ht]
\centering
\includegraphics[width=0.92\linewidth]{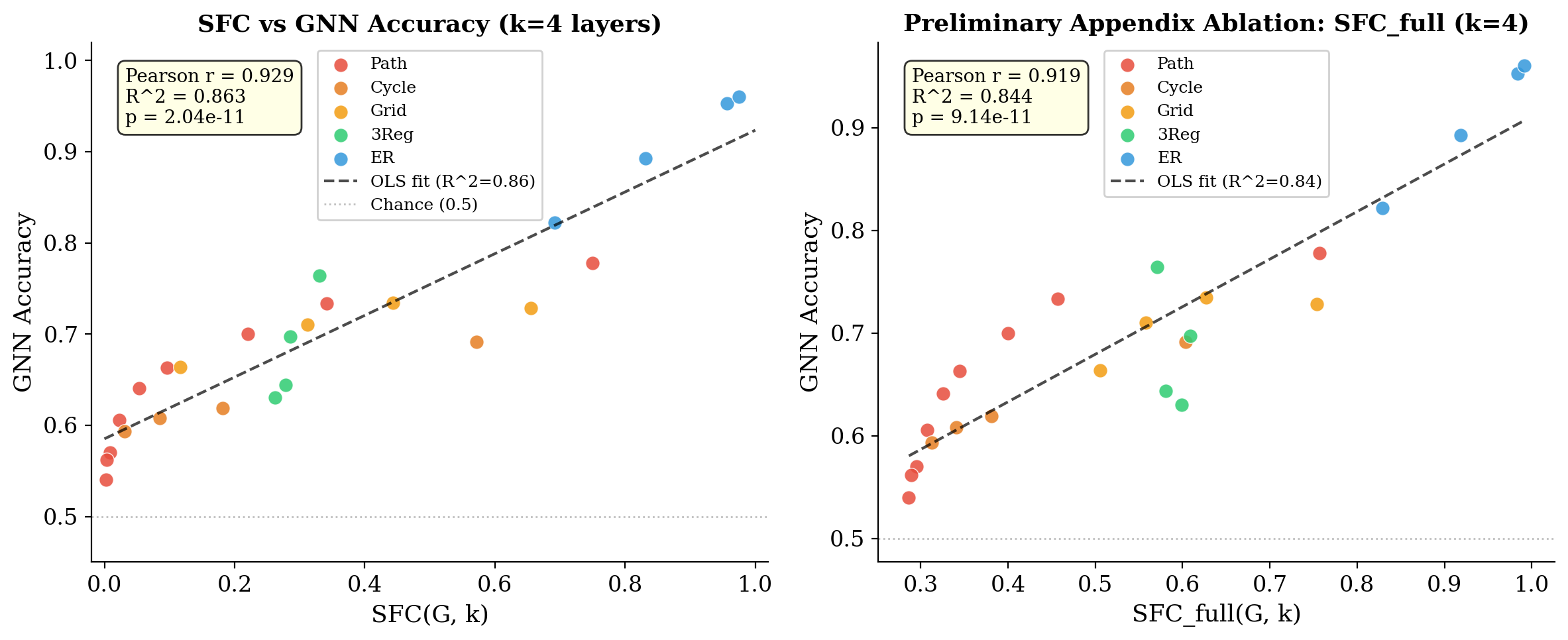}
\caption{Full synthetic experiment at $k=4$.}
\end{figure}

\begin{figure}[ht]
\centering
\includegraphics[width=0.92\linewidth]{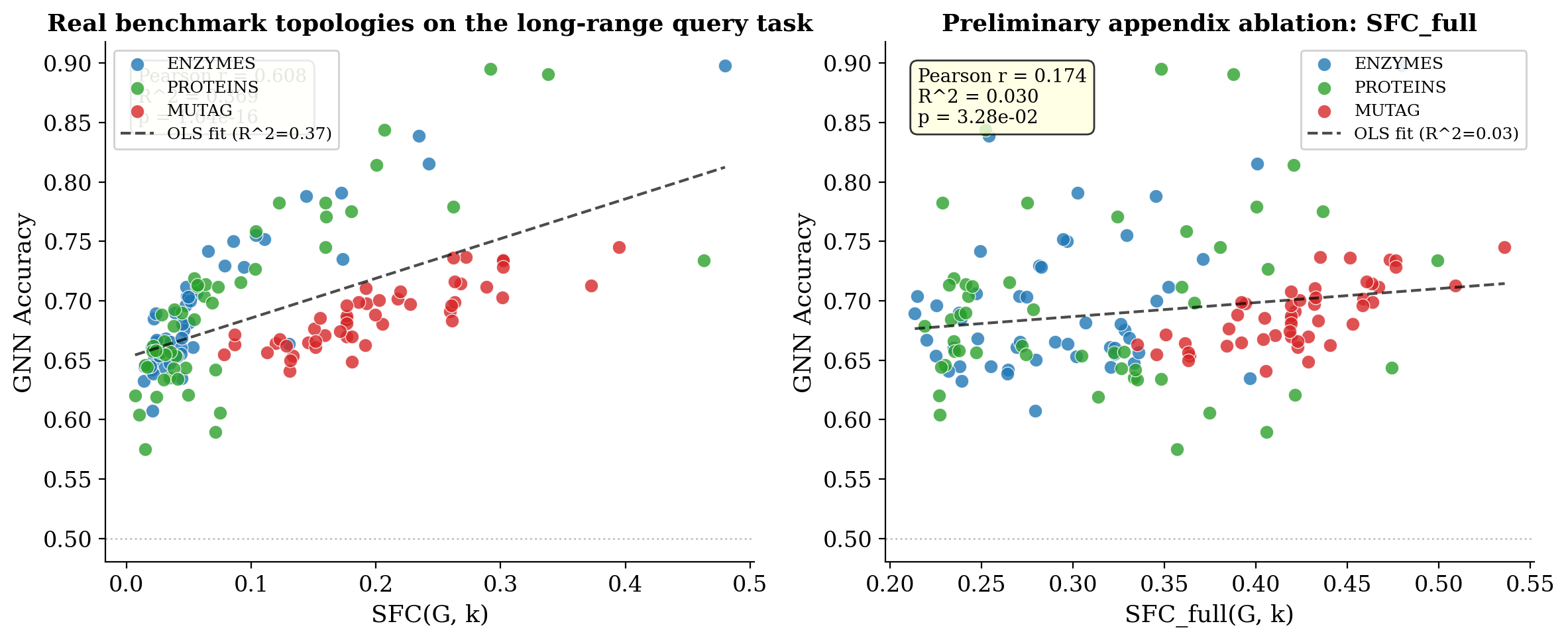}
\caption{Benchmark topologies, combined view.}
\end{figure}

\begin{figure}[ht]
\centering
\includegraphics[width=0.92\linewidth]{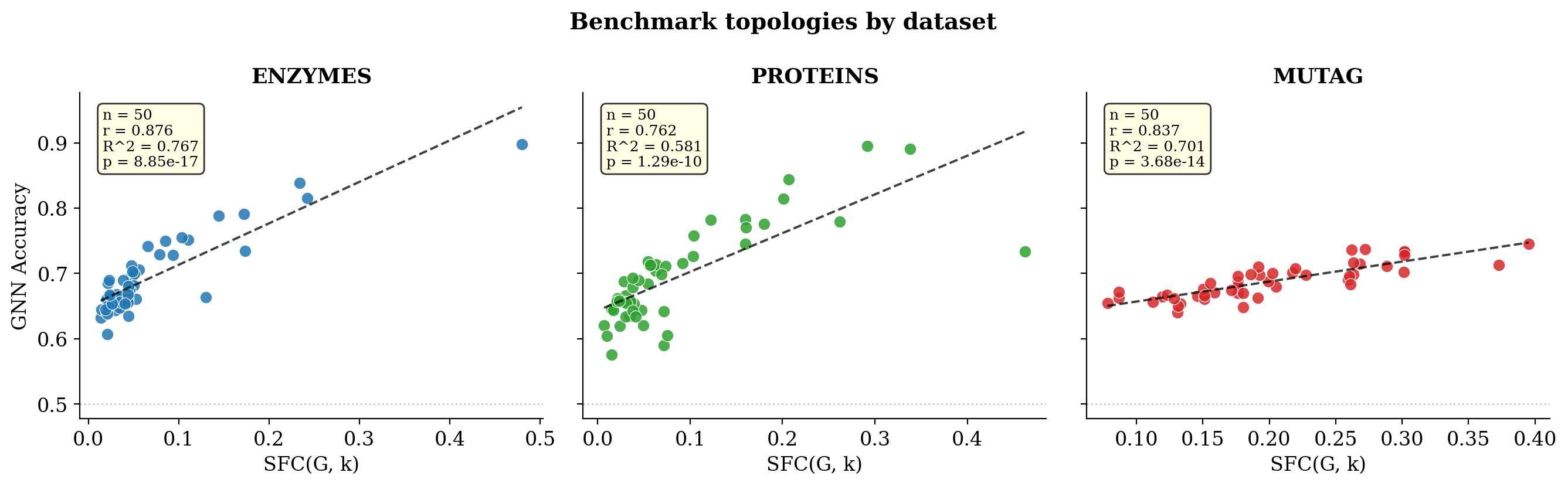}
\caption{Benchmark topologies, per-dataset panels (ENZYMES, PROTEINS, MUTAG).}
\end{figure}

\begin{figure}[ht]
\centering
\includegraphics[width=0.92\linewidth]{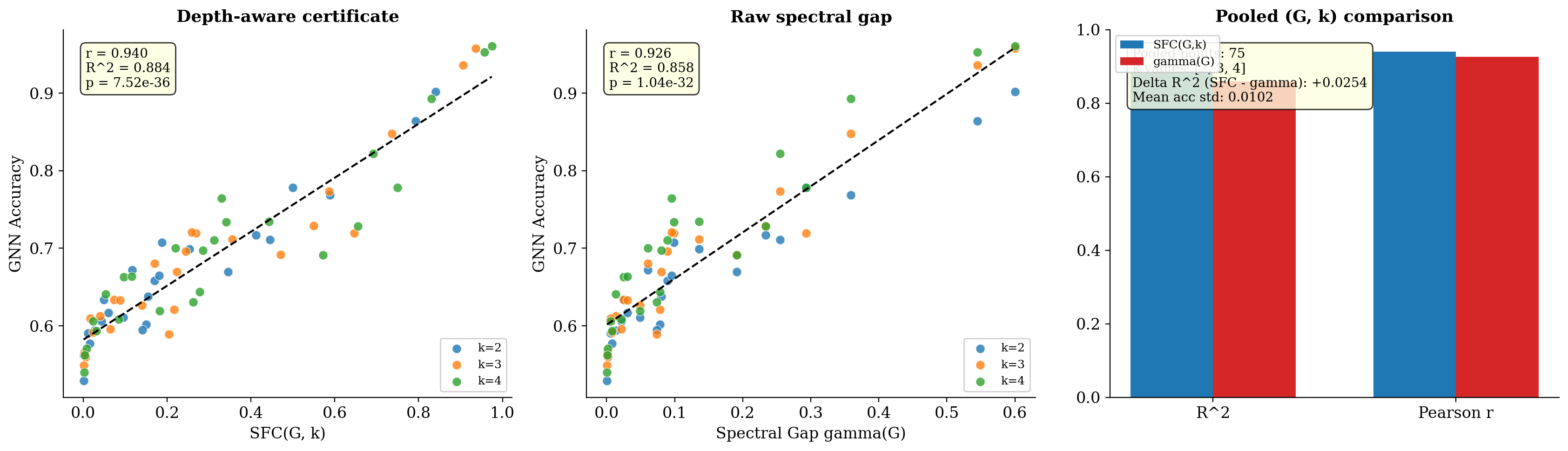}
\caption{Pooled depth comparison between SFC and raw spectral gap.}
\end{figure}

\begin{figure}[ht]
\centering
\includegraphics[width=0.92\linewidth]{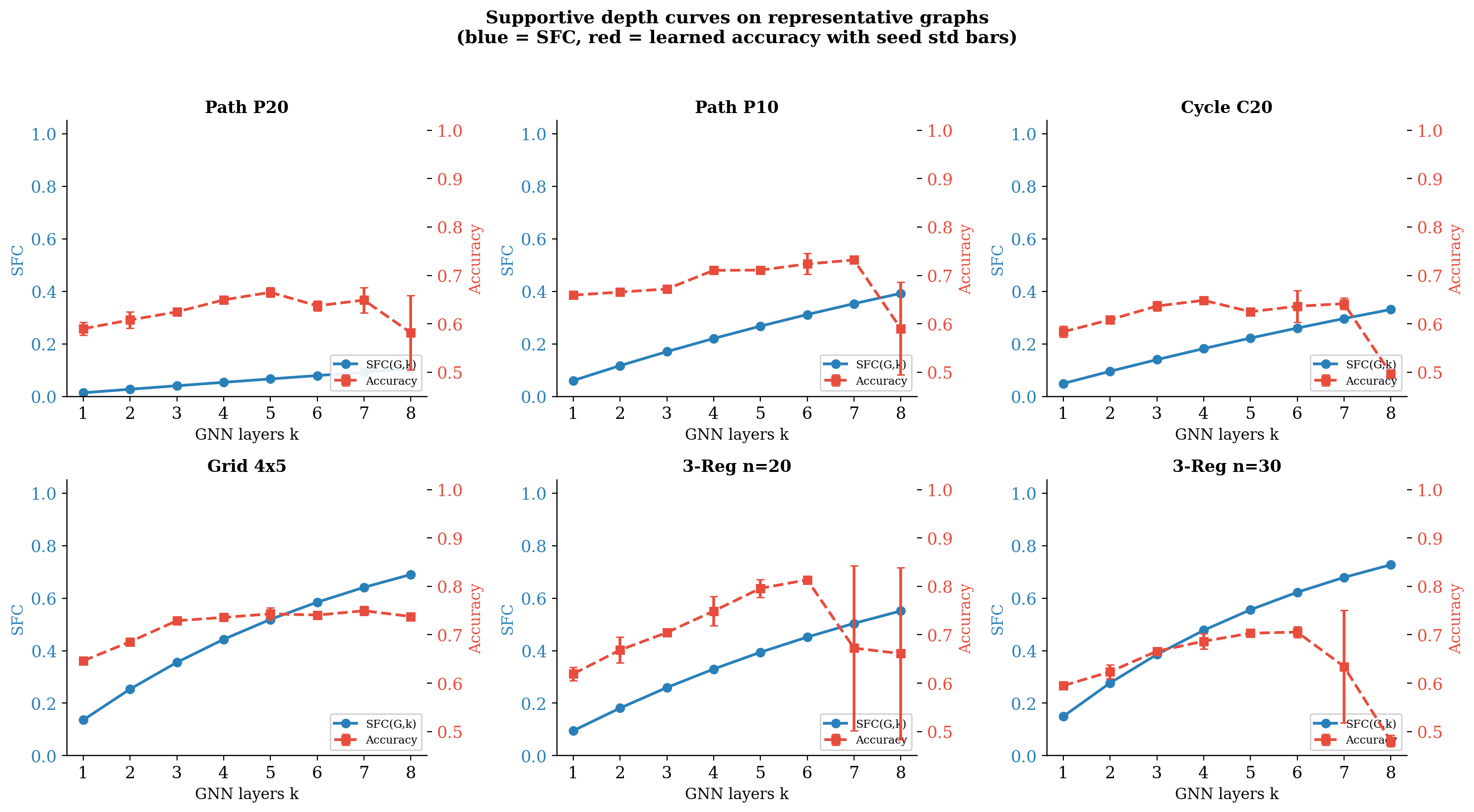}
\caption{Per-graph depth curves: SFC and accuracy versus $k$.}
\end{figure}

\begin{figure}[ht]
\centering
\includegraphics[width=0.92\linewidth]{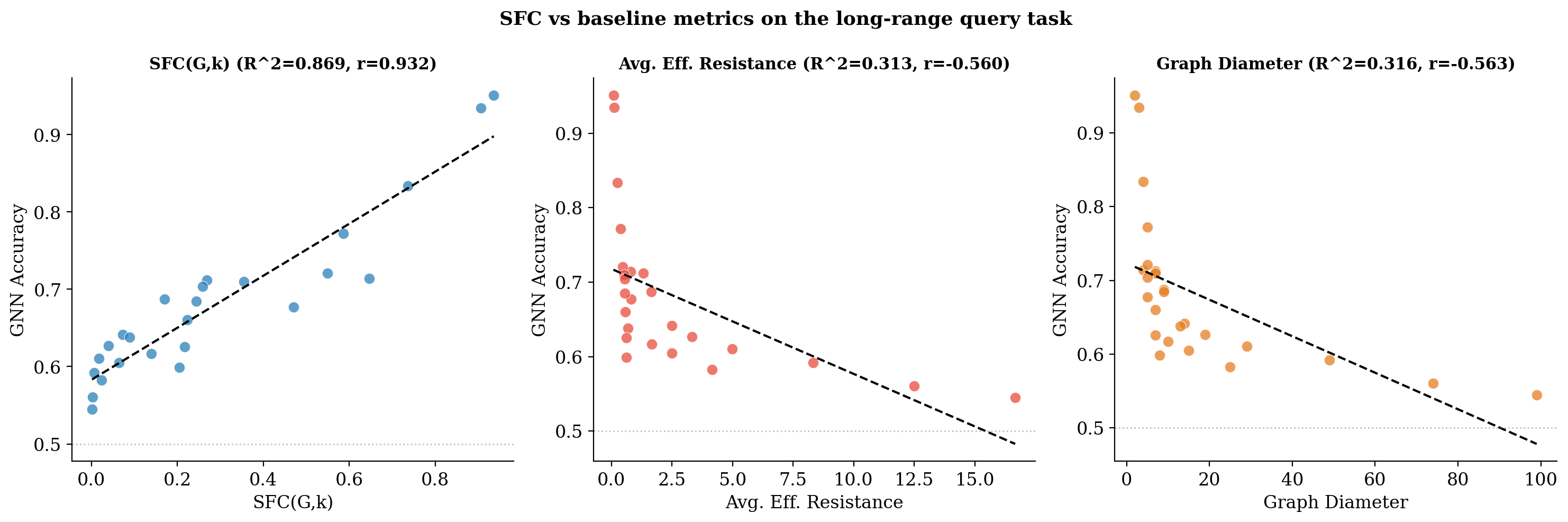}
\caption{Full non-spectral baseline comparison at $k=3$.}
\end{figure}

\begin{figure}[ht]
\centering
\includegraphics[width=0.92\linewidth]{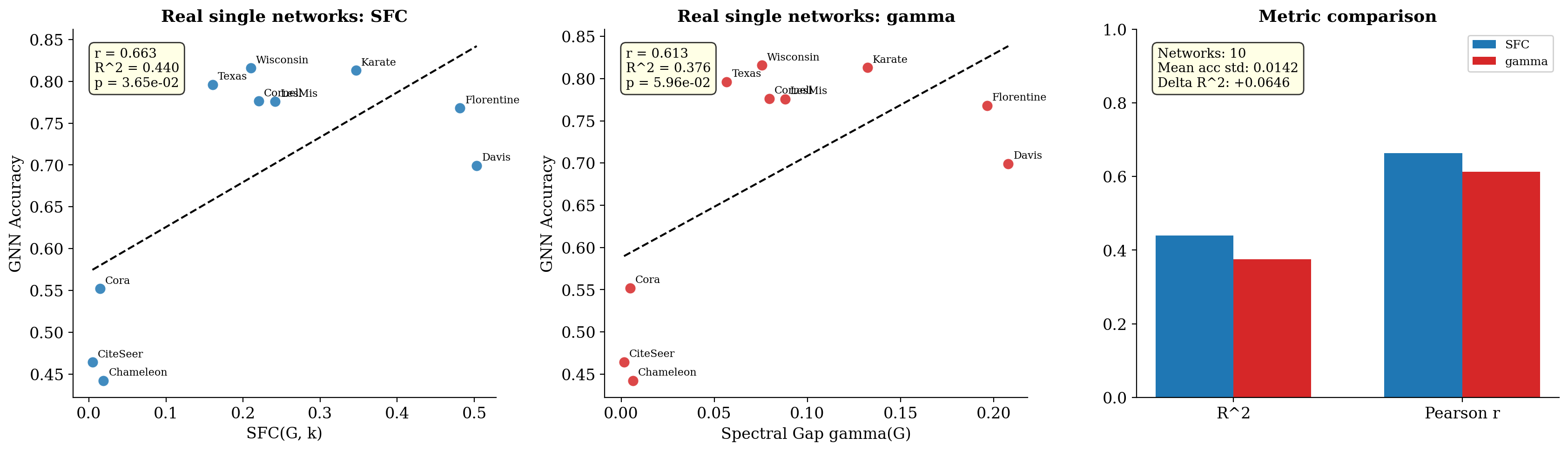}
\caption{Real single-network validation, pooled view.}
\end{figure}

\begin{figure}[ht]
\centering
\includegraphics[width=0.92\linewidth]{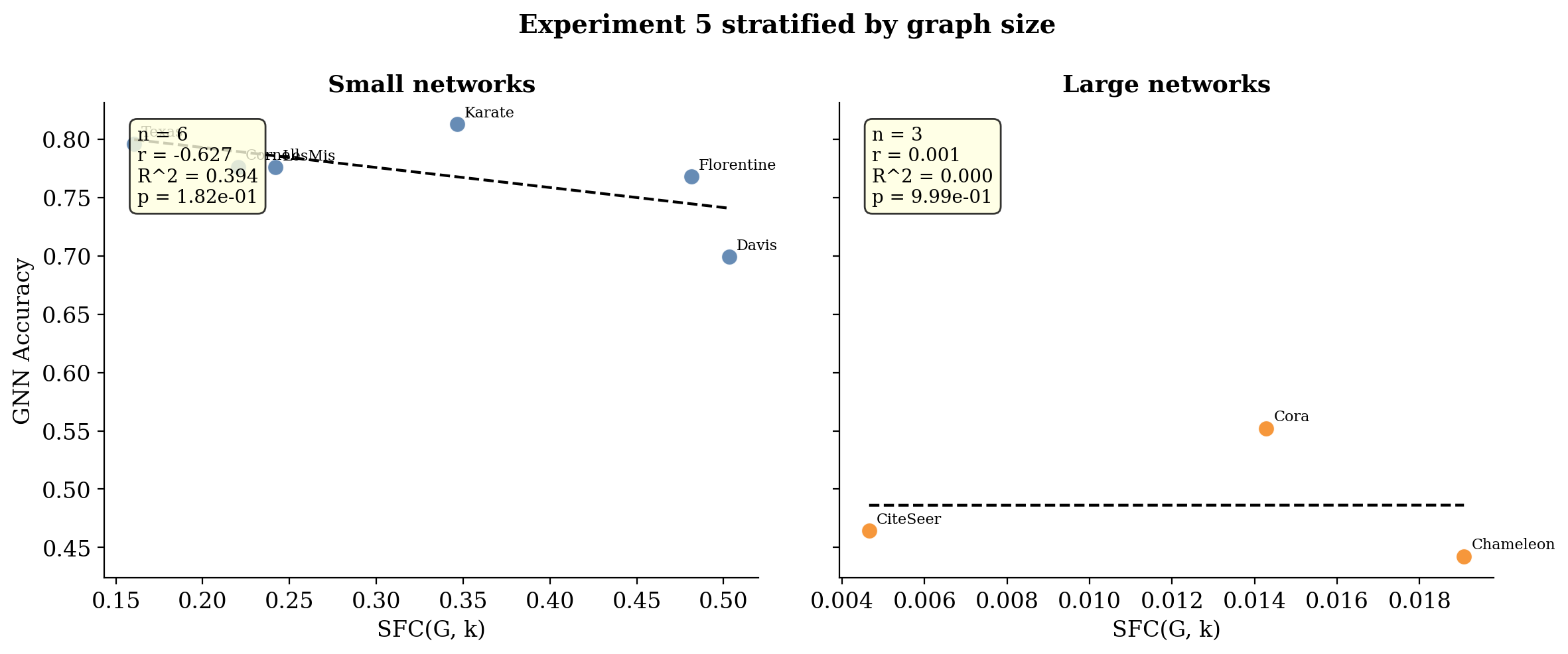}
\caption{Real single-network validation, size-stratified view.}
\end{figure}

\clearpage

\end{document}